%% file: main.tex
\documentclass[10pt]{article} 

\usepackage[preprint]{rlj} 

\usepackage{amssymb}            
\usepackage{mathtools}          
\usepackage{mathrsfs}           
\usepackage{graphicx}           
\usepackage{subcaption}         
\usepackage[space]{grffile}     
\usepackage{url}                
\usepackage{lipsum}             

\input{setup/packages}
\input{setup/math}

\input{setup/defs}

\title{Offline RL with Hierarchical Action Chunking}

\setrunningtitle{Offline RL with Hierarchical Action Chunking}

\author{Ahad Jawaid}

\emails{ahad.jawaid@utdallas.edu}

\affiliations{
\textbf{The University of Texas at Dallas}
}

\keywords{Reinforcement Learning,
Hierarchical Reinforcement Learning, Action Chunking} 

\begin{document}

\maketitle  

\input{content/abstract}


\input{content/body}

\appendix

\input{content/appendix}



\bibliography{main}
\bibliographystyle{rlj}

\beginSupplementaryMaterials

\input{content/supplementary}

\end{document}

%% file: setup/packages.tex
\usepackage{algorithm}
\usepackage{algorithmicx}
\usepackage[noEnd,commentColor=black]{algpseudocodex}

\usepackage{amsmath}
\usepackage{amssymb}
\usepackage{mathtools}
\usepackage{amsthm}
\usepackage{cancel}

\usepackage{quiver}

\usepackage{thmtools} 
\usepackage{thm-restate}

\usepackage{hyperref}       
\usepackage[capitalize,noabbrev]{cleveref}

\crefname{assumption}{Assumption}{Assumption}
\theoremstyle{plain}

\usepackage[textsize=tiny]{todonotes}

\usepackage[utf8]{inputenc} 
\usepackage[T1]{fontenc}    

\usepackage{url}            
\usepackage{booktabs}       

\usepackage{times}

\usepackage{xcolor} 
\usepackage{graphicx}
\usepackage{blindtext}
\usepackage[labelformat=empty, position=top]{subcaption}
\usepackage[export]{adjustbox}

\usepackage{wrapfig}
\usepackage{pifont}

\usepackage[utf8]{inputenc} 
\usepackage[T1]{fontenc}    
\usepackage{url}            
\usepackage{booktabs}       
\usepackage{amsfonts}       
\usepackage{nicefrac}       
\usepackage{microtype}      
\usepackage{xcolor}         

\usepackage{amsmath}
\usepackage{amssymb}
\usepackage{booktabs}
\usepackage{multirow}
\usepackage{multicol}
\usepackage{soul}
\usepackage{tikz-cd}

\usepackage{transparent}

\usepackage{tcolorbox}
\tcbuselibrary{theorems}
\usepackage{thmtools}

\definecolor{ourcolor}{HTML}{00A89D}
\definecolor{ourpurple}{HTML}{CB297B}
\definecolor{ourblue}{HTML}{0076BA}
\definecolor{ourmiddle}{HTML}{66509B}
\definecolor{ourlightblue}{HTML}{F5FAFE}
\definecolor{ourlightmiddle}{HTML}{F7F6F9}

\declaretheoremstyle[
    headfont=\bfseries,
    bodyfont=\normalfont,
    spaceabove=10pt, spacebelow=10pt,
    headpunct={},
    postheadspace=1em,
    mdframed={
        backgroundcolor=ourcolor!5,
        linecolor=ourcolor!50!black,
        linewidth=1pt,
        roundcorner=4pt
    }
]{colorboxed}

\declaretheoremstyle[
    headfont=\bfseries,
    bodyfont=\normalfont,
    spaceabove=10pt, spacebelow=10pt,
    headpunct={},
    postheadspace=1em,
    mdframed={
        backgroundcolor=ourblue!5,
        linecolor=ourblue!50!black,
        linewidth=1pt,
        roundcorner=4pt
    }
]{bluecolorboxed}

\declaretheoremstyle[
    headfont=\bfseries,
    bodyfont=\normalfont,
    spaceabove=10pt, spacebelow=10pt,
    headpunct={},
    postheadspace=1em,
    mdframed={
        backgroundcolor=ourpurple!5,
        linecolor=ourpurple!50!black,
        linewidth=1pt,
        roundcorner=4pt
    }
]{purplecolorboxed}

\declaretheoremstyle[
    headfont=\bfseries,
    bodyfont=\normalfont,
    spaceabove=10pt, spacebelow=10pt,
    headpunct={},
    postheadspace=1em,
    mdframed={
        backgroundcolor=black!5,
        linecolor=black!50!black,
        linewidth=1pt,
        roundcorner=4pt
    }
]{greycolorboxed}

\declaretheorem[name=Theorem,style=colorboxed]{theorem}
\declaretheorem[name=Lemma,style=purplecolorboxed]{lemma}

\declaretheorem[name=Assumption,style=bluecolorboxed]{assumption}

\usepackage{tikz}
\usetikzlibrary{decorations.pathreplacing,calc}

%% file: setup/math.tex
\usepackage{amsmath,amsfonts,bm, stackengine}

\newcommand{\res}[2]{%
    #1\,\stackengine{-0.3ex}{}{\scriptsize\textcolor{gray}{[#2]}}{O}{l}{F}{F}{L}%
}

\DeclarePairedDelimiterX{\infdivx}[2]{(}{)}{%
  #1\;\delimsize\|\;#2%
}

\def\eqref#1{equation~\ref{#1}}

\def\1{\bm{1}}

\DeclareMathAlphabet{\mathsfit}{\encodingdefault}{\sfdefault}{m}{sl}
\SetMathAlphabet{\mathsfit}{bold}{\encodingdefault}{\sfdefault}{bx}{n}



%% file: setup/defs.tex
\definecolor{darkblue}{rgb}{0,0.08,0.45}
\definecolor{cred}{HTML}{D62728}
\definecolor{cblue}{HTML}{1F77B4}
\definecolor{cgreen}{HTML}{79AB76}
\definecolor{cgrey}{rgb}{0.6,0.6,0.6}

\renewcommand{\mathbf}{\boldsymbol}

\makeatletter

%% file: content/abstract.tex
\begin{abstract}
Offline goal-conditioned reinforcement learning (RL) holds the promise of learning general-purpose policies from static datasets. However, scaling these methods to long-horizon tasks remains a challenge due to the ``curse of horizon'', where value estimation errors can compound through long chains of bootstrapped Bellman backups. Existing hierarchical approaches mitigate this by decomposing tasks into subgoals, yet they often rely on low-level controllers that suffer from myopic execution and biased value estimates. In this work, we propose \textbf{Hierarchical Implicit Q-Chunking (HiQC)}, an offline goal-conditioned RL algorithm that combines high-level latent planning with low-level action chunking. By conditioning the low-level critic on temporally extended action sequences, HiQC enables unbiased $k$-step value backups, compressing the horizon at both the planning and execution levels. We theoretically demonstrate that this dual decomposition results in a tighter bound on value error under a bounded per-backup error model compared to standard hierarchy or flat chunking alone. Empirically, HiQC achieves the highest aggregate performance among the compared methods on the OGBench suite, with its largest gains on long-horizon navigation tasks such as \texttt{humanoid-giant}.
\end{abstract}

%% file: content/body.tex
\begin{figure}[ht]
    \centering
    \includegraphics[width=\linewidth]{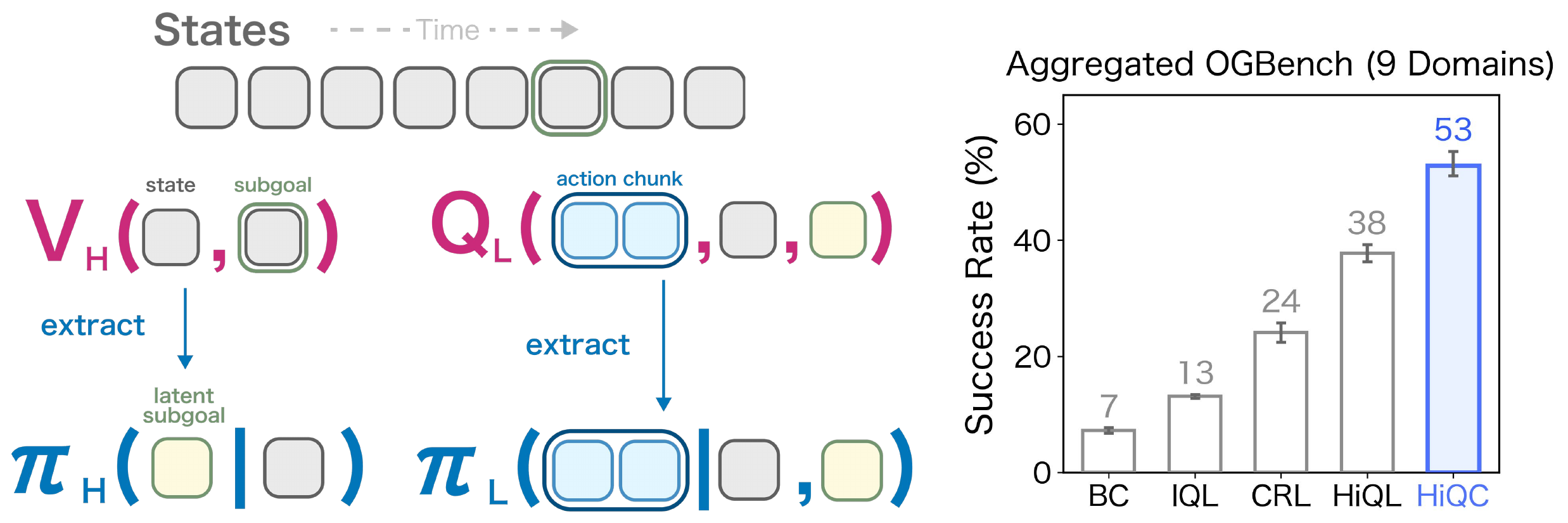}
    \caption{\textbf{Hierarchical Implicit Q-Chunking (HiQC).} HiQC addresses the curse of horizon in offline goal-conditioned RL through a two-level temporal decomposition. The \textbf{high-level planner} (left) operates in a latent space, predicting subgoals $z$ over a coarse horizon $c$ to bridge distant states. The \textbf{low-level controller} (right) executes action chunks $\mathbf{a}$ of length $k$ to reach these subgoals. By conditioning the low-level critic on full action chunks, HiQC enables unbiased $k$-step value backups, reducing the recursion depth required for value estimation while keeping low-level execution temporally consistent.}
    \label{fig:hero}
\end{figure}

\section{Introduction}

Offline reinforcement learning (RL) provides a framework for deriving optimal policies from static datasets, circumventing the need for expensive online interaction \citep{levine2020offline}. An obstacle to the success of offline RL in complex environments is the \textit{curse of horizon} \citep{liu2018breaking, park2025horizon}. Popular offline RL algorithms rely on temporal difference (TD) learning \citep{kostrikov2022iql}, which minimizes the error between a value estimate and a bootstrapped Bellman target \citep{sutton1988learning, sutton2018reinforcement}. In the offline setting, extrapolation errors on out-of-distribution actions introduce bias into this target \citep{kumar2020conservative}. As the task horizon increases, these biases can compound through long chains of recursive Bellman backups, rendering value estimates unreliable for distant goals \citep{park2025horizon}.

The bias accumulation can be mitigated by employing $n$-step returns, which skips $n$ steps of rewards before bootstrapping. However, as $n$ grows, this approach suffers from high variance, particularly in stochastic environments or when the data distribution deviates from the policy \citep{brandfonbrener2021offline}. Hierarchical reinforcement learning (HRL) offers an alternative mechanism for horizon reduction by decomposing tasks into sequences of high-level subgoals \citep{nachum2018data, park2023hiql}. Yet, standard hierarchies often rely on low-level policies that execute single actions to reach these subgoals. Consequently, the low-level policy remains susceptible to the curse of horizon over the duration of the subgoal, while the high-level policy depends on a potentially unstable low-level controller \citep{park2025horizon}.

Independently, \textit{action chunking} has emerged as a technique in imitation learning to assist in generating temporally coherent actions where policies predict sequences of future actions rather than single steps \citep{zhao2023learning}. Recent work has shown that within TD learning, action chunking allows the critic to perform unbiased $n$-step backups over the length of the chunk, effectively compressing the execution horizon \citep{li2025qchunking}. While action chunking is useful for low-level policies, it only reduces the horizon by a constant factor, which may be insufficient for long-horizon tasks.

In this work, we introduce \textbf{Hierarchical Implicit Q-Chunking (HiQC)}, an offline goal-conditioned RL algorithm designed to mitigate error accumulation through a dual horizon reduction on both the planning horizon and the execution horizon. In HiQC, the low-level policy operates directly on action chunks as the action space. This allows the low-level critic to utilize unbiased $k$-step returns. Simultaneously, the high-level policy decomposes the long-horizon task into a sequence of subgoals with a shorter horizon. By enforcing this chunk-based low-level policy, HiQC compresses the effective horizon at both the planning and execution levels, reducing the recursion depth required for value estimation.

Our main contributions operate at three distinct levels:
\begin{itemize}
    \item \textbf{Core structural idea.} We introduce \textbf{Hierarchical Implicit Q-Chunking (HiQC)}, an offline goal-conditioned RL algorithm whose central structure combines a latent hierarchical planner with a low-level controller that operates on action chunks. This design enables stable long-horizon learning by ensuring unbiased value propagation at the low level and compressed planning at the high level.
    \item \textbf{Theoretical analysis.} We analyze this structure under a bounded per-backup error model, showing that decomposing the task horizon $T$ into subgoals and action chunks of size $k$ yields a value estimation error bound scaling with $\mathcal{O}(\sqrt{T/k})$ rather than the $\mathcal{O}(T)$ of flat TD learning or the $\mathcal{O}(\sqrt{T})$ of standard hierarchies. This analysis is intended to build intuition for the benefit of the dual decomposition within a simplified bootstrap-chain model, rather than to provide end-to-end guarantees for the full algorithm.
    \item \textbf{Practical instantiation and empirical validation.} We instantiate the structure with expectile-based value learning, advantage-weighted regression, and a Conditional Flow Matching chunk policy, and evaluate it on OGBench \citep{park2025ogbench}. HiQC achieves the best aggregate score, with its largest gains on long-horizon navigation (e.g., \texttt{humanoid-giant}, \texttt{pointmaze-giant}).
\end{itemize}

\section{Related Work}

\textbf{Offline Reinforcement Learning.} Offline RL algorithms aim to learn optimal policies from static datasets without online interaction \citep{levine2020offline}. A challenge in this setting is the distributional shift between the learned policy and the behavioral policy, which leads to value overestimation for out-of-distribution actions \citep{kumar2020conservative}. To address this, prior methods have employed constraints on the policy \citep{fujimoto2019off}, conservative value estimation \citep{kumar2020conservative}, or in-sample advantage weighted regression such as Implicit Q-Learning (IQL) \citep{kostrikov2022offline}. While these methods have shown success, they often struggle in long-horizon tasks due to the accumulation of temporal difference (TD) errors over many time steps, a phenomenon characterized as the curse of horizon \citep{park2025horizon}. Our work builds upon the value learning stability of IQL but addresses the horizon bottleneck through hierarchical and low-level action abstractions.

\textbf{Hierarchical Reinforcement Learning.} Hierarchical RL (HRL) decomposes long-horizon problems into tractable sub-problems, typically formalized under the options framework \citep{sutton1999between}. In the context of goal-conditioned RL, this often manifests as a high-level policy generating subgoals for a low-level controller \citep{nachum2018data}. While early methods operated in the raw observation space, recent approaches have demonstrated the efficacy of planning over latent representations. Notably, Hierarchical Implicit Q-Learning (HIQL) \citep{park2023hiql} learns a two-level hierarchy from a single value function, where the high-level policy predicts latent subgoals and the low-level policy reaches them. Although HIQL effectively reduces the decision horizon for the high-level policy, its low-level controller operates at a single-step granularity, leaving it susceptible to local horizon effects and high-frequency actuation noise. HiQC extends this paradigm by integrating action chunking into the low-level controller, further compressing the effective horizon.

\textbf{Action Chunking in Learning.} Action chunking, the prediction of action sequences rather than single atomic actions, has gained prominence in imitation learning to model temporal correlations and mitigate compounding errors \citep{zhao2023learning, intelligence2025pi05, park2024hierarchical}. In the realm of reinforcement learning, Q-Chunking \citep{li2025qchunking} adapted this concept to TD learning, demonstrating that using action chunks on the critic allows for unbiased $n$-step backups. This effectively accelerates value propagation and encourages temporally coherent exploration. Variations such as Decoupled Q-Chunking \citep{li2025decoupled} have further refined this by separating the critic's chunk length from the policy's execution horizon. HiQC unifies these streams of research: we employ the latent goal-reaching structure of HIQL to handle global task decomposition into subgoals, while leveraging the temporal abstraction of Q-Chunking at the local level to ensure robust low-level execution and efficient value learning.

\section{Preliminaries}

\textbf{Problem Setting.} We consider the offline goal-conditioned reinforcement learning setting modeled as a Markov decision process (MDP) $\mathcal{M} = (\mathcal{S}, \mathcal{A}, p, r, \gamma)$, where $\mathcal{S}$ is the state space, $\mathcal{A}$ is the action space, $p(s'|s, a)$ is the transition dynamics, and $\gamma \in[0, 1)$ is the discount factor. We assume the goal space $\mathcal{G}$ coincides with the state space $\mathcal{S}$. The reward function $r(s, g)$ typically indicates binary success (e.g., $\mathbb{I}[s=g]$) or sparse distance. We are provided with a static dataset $\mathcal{D} = \{\tau^{(i)}\}_{i=1}^N$ consisting of trajectories $\tau = (s_0, a_0, s_1, \dots, s_T)$. The objective is to learn a goal-conditioned policy $\pi(a|s, g)$ that maximizes the expected cumulative discounted reward from $\mathcal{D}$ without further environment interaction.

\textbf{Implicit Q-Learning (IQL).} IQL \citep{kostrikov2022offline} is an offline RL algorithm that avoids querying values for out-of-distribution actions—a common source of value overestimation—by treating the Bellman maximization as an expectile regression problem. IQL learns a value function $V_\psi(s, g)$ and a Q-function $Q_\theta(s, a, g)$ by minimizing the following objectives:
\begin{align}
    L_V(\psi) &= \mathbb{E}_{s, a, g \sim \mathcal{D}} \left[ L_2^\tau \left( Q_{\hat{\theta}}(s, a, g) - V_\psi(s, g) \right) \right], \\
    L_Q(\theta) &= \mathbb{E}_{s, a, s', g \sim \mathcal{D}} \left[ \left( r(s, g) + \gamma V_\psi(s', g) - Q_\theta(s, a, g) \right)^2 \right],
\end{align}
where $L_2^\tau(u) = |\tau - \mathbb{I}(u < 0)|u^2$ is the expectile loss with parameter $\tau \in (0.5, 1)$. The value function $V_\psi$ approximates the expectile of the Q-values, which implicitly estimates the value of the best actions supported by the data. The policy is then extracted via advantage-weighted regression (AWR) \citep{peng2019advantage}.

\begin{wrapfigure}{r}{0.25\textwidth}
    \centering
    \includegraphics[width=\linewidth]{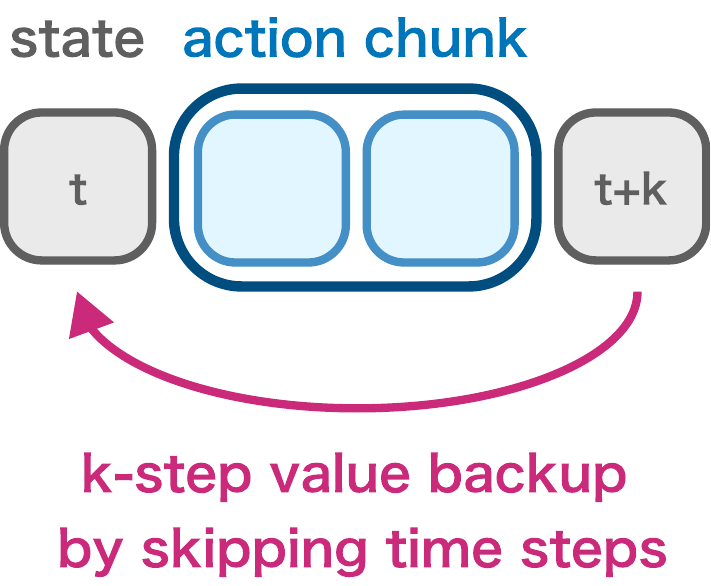}
    \vspace{-10pt}
    \caption{\small \textbf{$k$-step Value Backup.} Action chunks enable unbiased value propagation by skipping intermediate steps.}
    \label{fig:value_backup}
\end{wrapfigure}

\textbf{Q-Learning with Action Chunking.} Action chunking extends the standard MDP formulation by treating sequences of actions as atomic units. Let $\mathbf{a}_{t:t+k} = (a_t, a_{t+1}, \dots, a_{t+k-1}) \in \mathcal{A}^k$ denote an action chunk of length $k$. In this setting, the policy $\pi(\mathbf{a}_{t:t+k} | s_t)$ predicts a sequence of $k$ actions, which are executed open-loop. 

Standard approaches using $n$-step returns in offline RL often introduce bias because the intermediate actions in the dataset may not match the learned policy \citep{li2025qchunking}. Q-Chunking addresses this by training a critic $Q(s_t, \mathbf{a}_{t:t+k})$ conditioned on the entire action sequence. Because the critic evaluates the specific sequence present in the buffer, the $k$-step Bellman backup becomes unbiased with respect to that sequence (see Figure~\ref{fig:value_backup}):
\begin{equation}
    Q(s_t, \mathbf{a}_{t:t+k}) \leftarrow r(s_t) + \gamma^k V(s_{t+k}).
\end{equation}
This formulation allows for efficient value propagation over the horizon $k$ while avoiding the off-policy bias inherent in standard multi-step returns \citep{li2025qchunking}.

\section{Hierarchical Implicit Q-Chunking}
\label{sec:method}

We propose \textbf{Hierarchical Implicit Q-Chunking (HiQC)}, a hierarchical offline goal-conditioned RL algorithm designed to mitigate the curse of horizon by combining latent high-level planning with temporally extended low-level execution. HiQC decomposes the learning problem into two levels: a high-level (H) policy that plans a sequence of subgoals in a learned latent space over a horizon $c$, and a low-level (L) policy that executes action chunks of length $k$ to reach these subgoals.

This dual temporal decomposition is the core structure of HiQC, and it is the object of the theoretical analysis in Section~\ref{sec:theory}. To instantiate the structure as a practical algorithm, we make several implementation choices---expectile-based value learning \citep{kostrikov2022offline}, advantage-weighted policy extraction \citep{peng2019advantage}, and a flow-matching low-level policy \citep{lipmanflow}---which we describe below and assess separately in our ablation studies (Section~\ref{sec:experiments}).

\subsection{High-Level Learning: Latent Planning}

Following the architecture of HIQL \citep{park2023hiql}, we train a high-level value function $V_H(s, g)$ parametrized as $V_H(s, g) \approx w(s, \phi(g))$, where $\phi: \mathcal{S} \rightarrow \mathcal{Z}$ maps states to a latent subgoal space $\mathcal{Z}$. This parameterization simultaneously learns the value function and the representation $\phi$ required for planning.

The high-level value function estimates the discounted return of reaching the global goal $g$ from state $s$ with a temporal abstraction of $c$ steps (the subgoal horizon). We define the goal-conditioned reward function as $r(s, g) = \mathbb{I}(s=g) - 1$, where the agent receives $0$ upon reaching the goal and $-1$ otherwise. We train $V_H$ using Implicit V-Learning (IVL) \citep{kostrikov2022offline}, minimizing the expectile loss over $c$-step returns:
\begin{equation}
    \label{eq:vh_loss}
    \mathcal{L}_{V_H}(\psi) = \mathbb{E}_{\tau \sim \mathcal{D}, t} \left[ L_2^\tau \left( r(s_{t}, g) + \gamma^c V_H(s_{t+c}, g) - V_H(s_t, g) \right) \right].
\end{equation}
The high-level policy $\pi_H(z|s, g)$ is trained to predict the latent subgoal $z = \phi(s_{t+c})$ that leads to high-value states. We employ advantage-weighted regression (AWR) \citep{peng2019advantage} to maximize the likelihood of subgoals $z$ that have high advantage:
\begin{equation}
    \label{eq:pih_loss}
    \mathcal{L}_{\pi_H}(\vartheta) = \mathbb{E}_{\tau \sim \mathcal{D}, t} \left[ \exp\left( \frac{V_H(s_{t+c}, g) - V_H(s_t, g)}{\alpha} \right) \log \pi_H(\phi(s_{t+c}) | s_t, g) \right].
\end{equation}

\subsection{Low-Level Learning: Q-Chunking}

The low-level controller operates on action chunks $\mathbf{a} \in \mathcal{A}^k$, where $\mathbf{a} = (a_t, \dots, a_{t+k-1})$ represents a sequence of $k$ atomic actions. The goal of the low-level policy is to execute a chunk that reaches the latent subgoal $z$ specified by the high level. Note that the chunk size $k$ is an execution hyperparameter distinct from the planning horizon $c$. We define the low level subgoal-conditioned reward function as $r(s, z) = \mathbb{I}(s=z) - 1$, where the agent receives $0$ upon reaching the goal and $-1$ otherwise.

\textbf{Chunked Critic and Value.} To enable unbiased multi-step backups, we apply Q-Chunking \citep{li2025qchunking} to the low-level critic. We define a chunk-conditioned Q-function $Q_L(s, \mathbf{a}, z)$ and a state-value function $V_L(s, z)$. We train $V_L$ and $Q_L$ using IQL adapted for action chunks. The value function $V_L$ approximates the expectile of the chunk Q-values, while $Q_L$ regresses to a $k$-step target derived from the dataset transitions:
\begin{align}
    \label{eq:vl_loss}
    \mathcal{L}_{V_L}(\eta) &= \mathbb{E}_{s, \mathbf{a}, z \sim \mathcal{D}} \left[ L_2^\tau \left( Q_L(s, \mathbf{a}, z) - V_L(s, z) \right) \right], \\
    \label{eq:ql_loss}
    \mathcal{L}_{Q_L}(\theta) &= \mathbb{E}_{s, \mathbf{a}, s_{t+k}, z \sim \mathcal{D}} \left[ \left( r(s, z) + \gamma^k V_L(s_{t+k}, z) - Q_L(s, \mathbf{a}, z) \right)^2 \right].
\end{align}
Conditioning $Q_L$ on the full action chunk $\mathbf{a}$ removes the off-policy bias typically associated with $n$-step returns, allowing the low level to propagate values efficiently over the chunk horizon $k$ without divergence.

\textbf{Flow-Based Action Chunking Policy.} To model the complex, high-dimensional distribution of action sequences, we parameterize the low-level policy $\pi_L(\mathbf{a} | s, z)$ using Conditional Flow Matching (CFM) \citep{lipmanflow, zhangenergy}. The policy is defined by a vector field $v_\omega(t, x, s, z)$ that transforms noise $x_0 \sim \mathcal{N}(0, I)$ into an action chunk $\mathbf{a}$. We extract the optimal policy via advantage-weighted flow matching:
\begin{equation}
    \label{eq:pil_loss}
    \mathcal{L}_{\pi_L}(\omega) = \mathbb{E}_{t, x_t, s, \mathbf{a}, z \sim \mathcal{D}} \left[ \exp\left(\frac{Q_L(s, \mathbf{a}, z) - V_L(s, z)}{\beta}\right) \left\| v_\omega(t, x_t, s, z) - u_t(x|\mathbf{a}) \right\|^2 \right],
\end{equation}
where $u_t(x|\mathbf{a})$ is the target vector field and $\beta$ is a temperature parameter.

\begin{algorithm}[ht]
\caption{Hierarchical Implicit Q-Chunking (HiQC)}
\label{alg:hiqc}
\begin{algorithmic}[1]
\State \textbf{Input:} Dataset $\mathcal{D}$, subgoal horizon $c$, chunk size $k$.
\State \textbf{Initialize:} $V_H, \pi_H, V_L, Q_L, \pi_L$.
\While{training}
    \State Sample batch of trajectories $\tau \sim \mathcal{D}$.
    \State \textbf{// High-Level Update (Planning Horizon $c$)}
    \State Sample time steps $t$. Set current state $s_t$, goal $g$.
    \State Compute latent target $z = \phi(s_{t+c})$.
    \State Update $V_H$ via Implicit V-Learning using Eq.~\ref{eq:vh_loss}.
    \State Update $\pi_H$ via AWR using Eq.~\ref{eq:pih_loss}.
    \State \textbf{// Low-Level Update (Execution Horizon $k$)}
    \State Extract action chunk $\mathbf{a} = a_{t:t+k}$.
    \State Update $V_L$ via Chunked IQL using Eq.~\ref{eq:vl_loss}.
    \State Update $Q_L$ via $k$-step regression using Eq.~\ref{eq:ql_loss}.
    \State Update $\pi_L$ via Advantage-Weighted Flow Matching using Eq.~\ref{eq:pil_loss}.
\EndWhile
\State \textbf{Inference:}
\State Given current state $s$ and global goal $g$:
\State \quad 1. Sample latent subgoal $z \sim \pi_H(\cdot | s, g)$.
\State \quad 2. Generate action chunk $\mathbf{a} \sim \pi_L(\cdot | s, z)$.
\State \quad 3. Execute $\mathbf{a}$ open-loop for $k$ steps.
\State \quad 4. Repeat (High-level planning may recur at frequency $1/c$ or $1/k$).
\end{algorithmic}
\end{algorithm}

\section{Theoretical Analysis}
\label{sec:theory}

We analyze HiQC using a \emph{bootstrap-chain} error model: value learning propagates information backward by repeatedly applying bootstrapped targets, and each bootstrap step can introduce bounded error. The key quantity is therefore the \emph{number of bootstraps} required to propagate information over the horizon.

\subsection{Bootstrap-Chain Error Model}

Let $T$ be the task horizon in \emph{atomic} environment steps. Consider any value-learning scheme that propagates values backward through a sequence of bootstrapped updates. We define:

\textbf{Bootstrap depth.} Let $D$ denote the number of bootstraps (Bellman-style backups) that must be chained to propagate terminal signal back to the start state under a given temporal abstraction.

\begin{assumption}[Bounded one-backup error]
\label{ass:one_backup}
For a given value recursion, each bootstrap step introduces at most $\epsilon$ value error in magnitude. For hierarchical methods we use $\epsilon_H$ for the high-level recursion and $\epsilon_L$ for the low-level recursion.
\end{assumption}

Under Assumption~\ref{ass:one_backup}, the value error grows at most linearly with the number of chained bootstraps.

\begin{lemma}[Error grows with bootstrap depth]
\label{lem:depth}
If a value estimate is obtained by chaining $D$ bootstraps and each bootstrap contributes at most $\epsilon$ error (Assumption~\ref{ass:one_backup}), then the resulting value error is bounded by
\begin{equation}
\label{eq:depth_bound}
\mathcal{E} \;\le\; D\,\epsilon.
\end{equation}
\end{lemma}

Lemma~\ref{lem:depth} is proved by a one-line unrolling argument in Appendix~\ref{app:proofs}. We now compute $D$ for TD learning, flat chunking, standard hierarchy, and HiQC.

\subsection{Baselines}

For readability we assume $T$ is divisible by $k$ and $c$, and $c$ is divisible by $k$.\footnote{If not divisible, the same bounds hold up to rounding by replacing ratios with $\lceil\cdot\rceil$. This does not change the scaling and is omitted for clarity.}

\paragraph{Standard TD learning (IQL).}
TD bootstraps every atomic step, so the chain length is $D=T$. By Lemma~\ref{lem:depth},
\begin{equation}
\label{eq:td}
\mathcal{E}_{\text{TD}}(T) \;\le\; T\,\epsilon.
\end{equation}

\paragraph{Flat action chunking (QC).}
With chunk size $k$, the critic bootstraps every $k$ steps, reducing the number of bootstraps to $D=T/k$. By Lemma~\ref{lem:depth},
\begin{equation}
\label{eq:qc}
\mathcal{E}_{\text{QC}}(T,k) \;\le\; \frac{T}{k}\,\epsilon.
\end{equation}

\paragraph{Standard hierarchy (HIQL-style).}
With subgoal spacing $c$, the high level bootstraps every $c$ steps, so its chain length is $D_H=T/c$ and contributes $(T/c)\epsilon_H$. The low level must bridge $c$ atomic steps using one-step bootstraps, so $D_L=c$ and contributes $c\,\epsilon_L$. Summing the two gives
\begin{equation}
\label{eq:hiql}
\mathcal{E}_{\text{HIQL}}(T,c)
\;\le\;
\frac{T}{c}\,\epsilon_H \;+\; c\,\epsilon_L.
\end{equation}

\subsection{HiQC}

HiQC uses the same high-level recursion as HIQL (bootstrapping every $c$ steps), but its low-level critic bootstraps over \emph{chunks} of length $k$. To bridge a subgoal interval of length $c$, the low level therefore chains $D_L=c/k$ chunk-bootstraps (rather than $c$ one-step bootstraps). Applying Lemma~\ref{lem:depth} at each level yields:

\begin{theorem}[HiQC bootstrap-chain bound]
\label{thm:hiqc}
Under Assumption~\ref{ass:one_backup}, HiQC satisfies
\begin{equation}
\label{eq:hiqc}
\mathcal{E}_{\text{HiQC}}(T,c,k)
\;\le\;
\frac{T}{c}\,\epsilon_H \;+\; \frac{c}{k}\,\epsilon_L.
\end{equation}
Moreover, the right-hand side is minimized at
\begin{equation}
\label{eq:c_star}
c^\star \;=\; \sqrt{\frac{Tk\,\epsilon_H}{\epsilon_L}},
\end{equation}
with minimum value
\begin{equation}
\label{eq:hiqc_opt}
\min_c \mathcal{E}_{\text{HiQC}}
\;\le\;
2\sqrt{\epsilon_H\epsilon_L}\,\sqrt{\frac{T}{k}}
\;=\;
\mathcal{O}\!\left(\sqrt{\frac{T}{k}}\right).
\end{equation}
\end{theorem}

The proof is given in Appendix~\ref{app:proofs}. Intuitively, chunking reduces the number of low-level bootstraps needed per subgoal interval from $c$ to $c/k$, which improves the optimized scaling from $\mathcal{O}(\sqrt{T})$ to $\mathcal{O}(\sqrt{T/k})$ in this bootstrap-depth model.

\textbf{Scope of the analysis.} The bootstrap-chain model abstracts away function approximation, sampling error, and the interaction between the two levels during learning, and Assumption~\ref{ass:one_backup} may not hold uniformly when values are represented by deep neural networks. Moreover, the analysis concerns the structure of the value recursion and is agnostic to the implementation choices of Section~\ref{sec:method} (e.g., expectile regression or the flow-based policy). We therefore view Theorem~\ref{thm:hiqc} as building intuition for why HiQC's dual decomposition reduces error accumulation, rather than as an end-to-end guarantee for the trained system evaluated in Section~\ref{sec:experiments}.

\section{Experiments}
\label{sec:experiments}

In this section, we empirically evaluate Hierarchical Implicit Q-Chunking (HiQC) to determine if the combination of high-level latent planning and low-level action chunking mitigates the curse of horizon. We focus our analysis on three questions: (1) Does HiQC outperform standard hierarchical and flat baselines on long-horizon navigation tasks? (2) How does it compare to recent offline GCRL methods on high-dimensional manipulation tasks? (3) What is the impact of key design choices, such as the flow-based policy parameterization and chunk size?

\subsection{Experimental Setup}

\textbf{Environments.} We evaluate our method on the OGBench suite \citep{park2025ogbench}, a benchmark designed to stress-test offline GCRL algorithms. We select 9 tasks covering two distinct domains: \textit{locomotion} (\texttt{antmaze}, \texttt{humanoidmaze}, \texttt{pointmaze}) and \textit{manipulation} (\texttt{cube}, \texttt{scene}, \texttt{puzzle}). We do not use the oracle state representation in any of our evaluations. The reward is sparse as defined in our preliminaries. Furthermore, for the low-level policy, the reward function retains the same sparse structure but evaluates reaching the generated subgoal rather than the overarching task goal. To test horizon capabilities, we include the ``giant'' variants of navigation tasks (e.g., \texttt{humanoidmaze-giant}), which require reasoning over thousands of time steps and have been shown to break prior methods \citep{park2025horizon}.

\textbf{Baselines.} We compare HiQC against a diverse set of baselines representing three paradigms:
\begin{itemize} \setlength\itemsep{0em}
    \item \textbf{Flat Baselines:} Goal-Conditioned Behavioral Cloning (GCBC), Contrastive RL (CRL) \citep{eysenbach2022contrastive}, and Goal-Conditioned Implicit Q-Learning (GCIQL) \citep{kostrikov2022offline}.
    \item \textbf{Action Chunking Baselines:} Q-Chunking (QC) \citep{li2025qchunking} and Decoupled Q-Chunking (DQC) \citep{li2025decoupled}. These methods utilize temporal abstraction but lack explicit high-level planning.
    \item \textbf{Hierarchical Baselines:} Hierarchical GCBC (HGCBC) \citep{park2025horizon}, HGCBC with Action Chunking (HGCBCAC), and Hierarchical Implicit Q-Learning (HIQL) \citep{park2023hiql}. We implement HGCBCAC  by extending HGCBC to predict action chunks at the low level.
\end{itemize}

\textbf{Implementation Details.} We trained all methods for 1 million gradient steps with a batch size of 1024. Experiments were conducted on NVIDIA RTX 4090 GPUs, with each run taking approximately 1.5 hours. We report results averaged over 4 random seeds. HiQC inherits the default hyperparameters from HIQL \citep{park2023hiql} for the high-level components and from QC \citep{li2025qchunking} for the low-level chunking components, with limited sensitivity checks on the \texttt{scene} environments (Figure~\ref{fig:hyperparams}). All policies and value functions were parameterized as Multi-Layer Perceptrons (MLPs). Full architectural details, the domain-specific $(c, k, \tau)$ values, the \texttt{scene-*} stabilization settings, and a discussion of the seed count and tuning budget are provided in Supplementary Material \ref{app:hyperparameters}.

\begin{figure}[t]
    \centering
    \includegraphics[width=\linewidth]{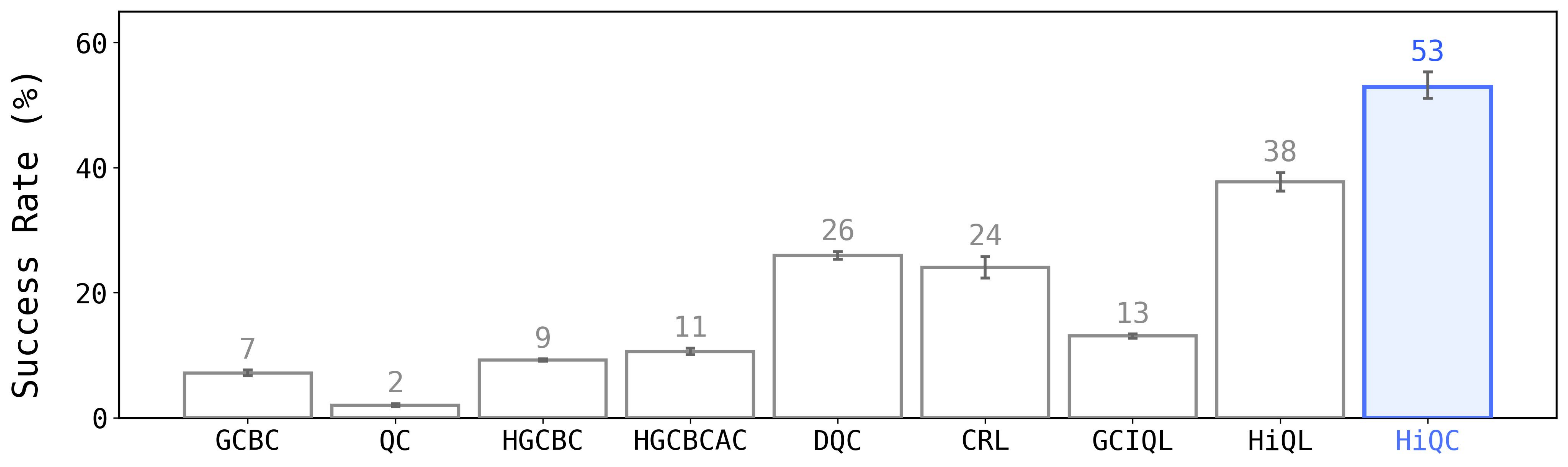}
    \caption{\textbf{Aggregated Performance on OGBench.} We report the mean success rate across 9 diverse domains.}
    \label{fig:benchmark}
\end{figure}

\begin{table*}[t]
\centering
\caption{\textbf{Success rates on OGBench tasks.} We report the mean success rate  95\% confidence interval over 4 seeds. The best performing method for each task is highlighted in \textbf{bold}.}
\label{tab:main_results}
\small
\setlength{\tabcolsep}{3pt} 
\begin{tabular}{l lllllllll}
\toprule
\textbf{Task} & \textbf{BC} & \textbf{QC} & \textbf{HBC} & \textbf{HBCAC} & \textbf{DQC} & \textbf{CRL} & \textbf{GCIQL} & \textbf{HiQL} & \textbf{HiQC} \\
\midrule
antmaze-large    & \res{27}{24,30} & \res{0}{0,0}    & \res{49}{48,51} & \res{54}{52,56} & \res{35}{33,38} & \res{81}{78,84} & \res{32}{25,37} & \res{88}{86,89} & \textbf{\res{93}{92,94}} \\
antmaze-giant    & \res{0}{0,0}    & \res{0}{0,0}    & \res{25}{22,27} & \res{30}{27,33} & \res{1}{0,1}    & \res{14}{11,18} & \res{0}{0,1}    & \textbf{\res{68}{67,69}} & \textbf{\res{68}{67,69}} \\
humanoid-large   & \res{1}{1,2}    & \res{0}{0,0}    & \res{4}{3,5}    & \res{6}{5,7}    & \res{1}{1,1}    & \res{17}{12,23} & \res{1}{1,2}    & \res{34}{23,44} & \textbf{\res{42}{39,44}} \\
humanoid-giant   & \res{0}{0,0}    & \res{0}{0,0}    & \res{1}{0,1}    & \res{0}{0,0}    & \res{0}{0,0}    & \res{4}{2,5}    & \res{0}{0,0}    & \res{10}{7,15}  & \textbf{\res{33}{29,39}} \\
pointmaze-large  & \res{29}{25,34} & \res{15}{13,17} & \res{0}{0,0}    & \res{0}{0,1}    & \res{62}{60,65} & \res{36}{32,41} & \res{29}{24,33} & \res{47}{39,55} & \textbf{\res{87}{81,92}} \\
pointmaze-giant  & \res{1}{0,1}    & \res{0}{0,0}    & \res{0}{0,0}    & \res{0}{0,0}    & \res{35}{29,40} & \res{33}{27,39} & \res{1}{0,3}    & \res{44}{35,53} & \textbf{\res{80}{76,85}} \\
cube-triple      & \res{1}{1,2}    & \res{0}{0,0}    & \res{0}{0,0}    & \res{0}{0,0}    & \textbf{\res{8}{7,10}}   & \res{4}{3,4}    & \res{1}{1,2}    & \res{3}{2,4}    & \textbf{\res{9}{4,15}} \\
puzzle-4x6       & \res{0}{0,0}    & \res{0}{0,0}    & \res{0}{0,0}    & \res{0}{0,0}    & \textbf{\res{16}{14,17}} & \res{5}{4,6}    & \res{5}{4,6}    & \res{4}{4,5}    & \res{1}{1,2} \\
scene-play       & \res{5}{4,5}    & \res{3}{2,3}    & \res{4}{4,5}    & \res{5}{5,6}    & \textbf{\res{76}{75,77}} & \res{22}{21,23} & \res{47}{44,49} & \res{41}{38,44} & \res{65}{61,68} \\
\midrule
\textbf{Overall} & \res{7}{6,8}    & \res{2}{2,2}    & \res{9}{8,10}   & \res{11}{10,11} & \res{26}{23,28} & \res{23}{21,26} & \res{13}{11,15} & \res{42}{37,47} & \textbf{\res{53}{50,56}} \\
\bottomrule
\end{tabular}
\end{table*}

\subsection{Comparative Analysis on OGBench}

We present the main comparative results in Table \ref{tab:main_results} and summarize the aggregate performance in Figure \ref{fig:benchmark}.

\textbf{Long-Horizon Navigation.} HiQC shows its largest gains in long-horizon locomotion tasks. In \texttt{pointmaze-giant} and \texttt{humanoid-giant}, which require traversing paths spanning thousands of time steps, flat methods (GCBC, QC, GCIQL) fail almost completely, illustrating the curse of horizon for flat value functions \citep{park2025horizon}. While standard HIQL achieves non-trivial performance on \texttt{humanoid-giant} ($10\%$), HiQC outperforms it ($33\%$). This result is consistent with our theoretical analysis (Theorem \ref{thm:hiqc}): by utilizing unbiased $k$-step backups in the low-level critic, HiQC compresses the effective horizon, reducing the number of recursive Bellman backups required for value propagation, which typically destabilizes hierarchical execution in giant mazes. On \texttt{antmaze-large}, HiQC ($93\%$) surpasses both HIQL ($88\%$) and CRL ($81\%$). On \texttt{antmaze-giant}, however, HiQC matches HIQL ($68\%$) rather than improving upon it; we note that \texttt{antmaze} uses the smallest chunk size among our domains ($k=2$, Supplementary Material \ref{app:hyperparameters}), which limits the amount of horizon compression available at the low level.

\textbf{High-Dimensional Manipulation.} On manipulation tasks, which require precise sequential object interactions, the results are mixed. HiQC attains the best success rate on \texttt{cube-triple} ($9\%$, comparable to DQC's $8\%$), outperforming QC ($0\%$) and HIQL ($3\%$), and it is the second-best method on \texttt{scene-play} ($65\%$), ahead of HIQL ($41\%$) and GCIQL ($47\%$). However, Decoupled Q-Chunking (DQC) outperforms HiQC on \texttt{scene-play} ($76\%$) and \texttt{puzzle-4x6} ($16\%$ vs.\ $1\%$). A structural difference between the two methods is that DQC decouples its critic horizon from its execution horizon: on these tasks, its critic performs $25$-step backups while its policy executes $5$-step chunks, whereas HiQC's low-level critic horizon equals its executed chunk size ($k=5$) and relies on the hierarchy---through a learned latent subgoal space---for longer-range value propagation. We hypothesize that this decoupling underlies DQC's advantage on these tasks; we examine this performance gap in more detail in Supplementary Material~\ref{app:dqc_comparison}. Overall, these results indicate that HiQC's advantage is concentrated in long-horizon navigation, and its leading aggregate score should not be read as uniform superiority across domains.

\subsection{Ablation Studies}

\begin{wrapfigure}{r}{0.30\textwidth}
    \vspace{-15pt}
    \centering
    \includegraphics[width=\linewidth]{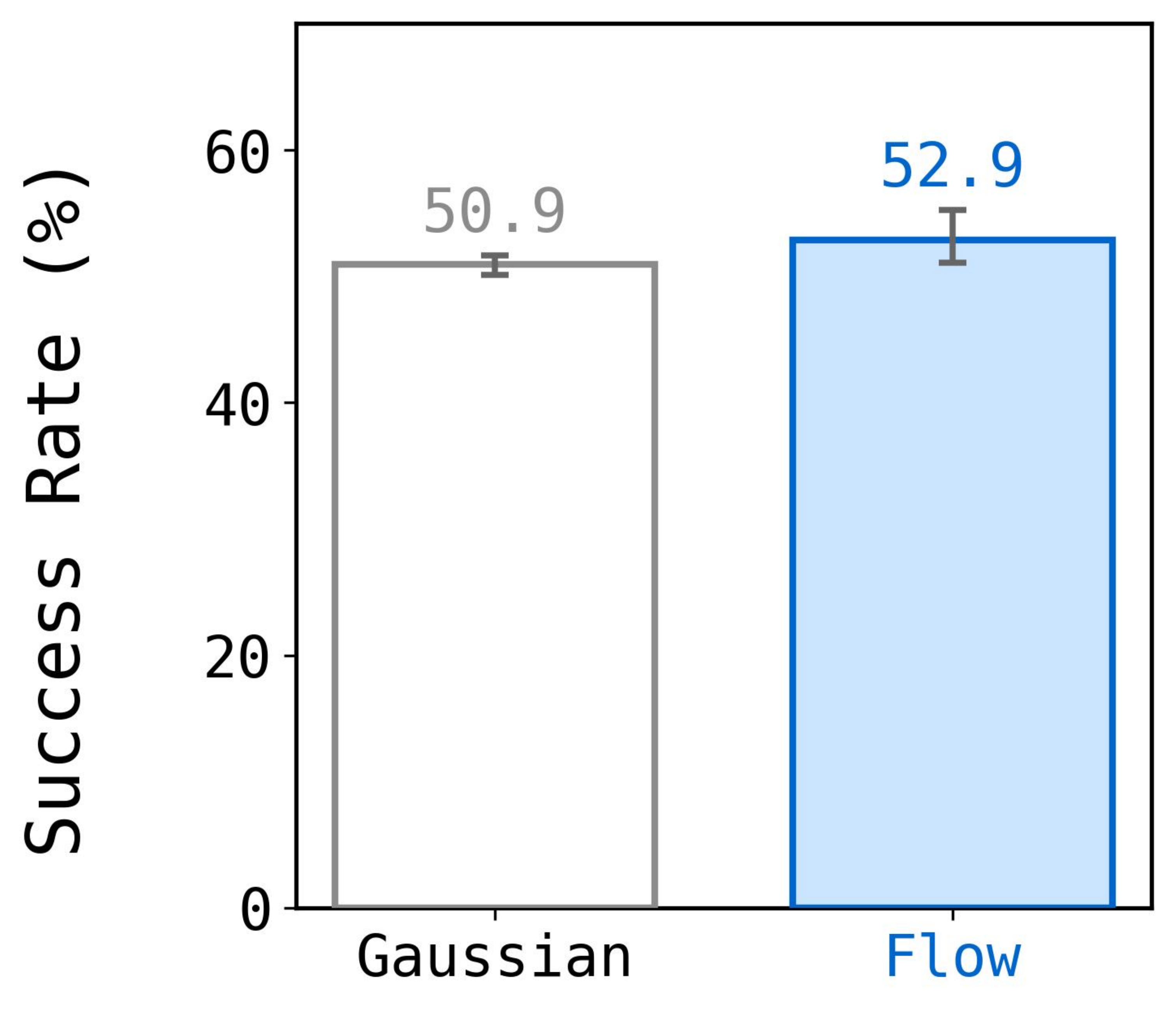}
    \vspace{-15pt}
    \caption{\textbf{Flow vs. Gaussian Policy Ablation.}}
    \label{fig:ablation_flow}
\end{wrapfigure}

\textbf{Impact of Flow Matching Policies.} To assess the impact of our policy parameterization, we compare HiQC using a standard Gaussian policy against our default Conditional Flow Matching (CFM) policy. As shown in Figure \ref{fig:ablation_flow}, the flow-based policy consistently achieves higher success on 9 environments from OGBench. This aligns with prior findings in imitation learning \citep{zhao2023learning} and Q-learning \citep{li2025qchunking}, suggesting that flow matching better captures the multimodal distributions inherent in trajectory data, preventing the ``averaging'' artifacts of Gaussian policies that can disrupt temporally extended actions.

\begin{figure}[ht]
    \vspace{15pt}
    \centering
    \includegraphics[width=\linewidth]{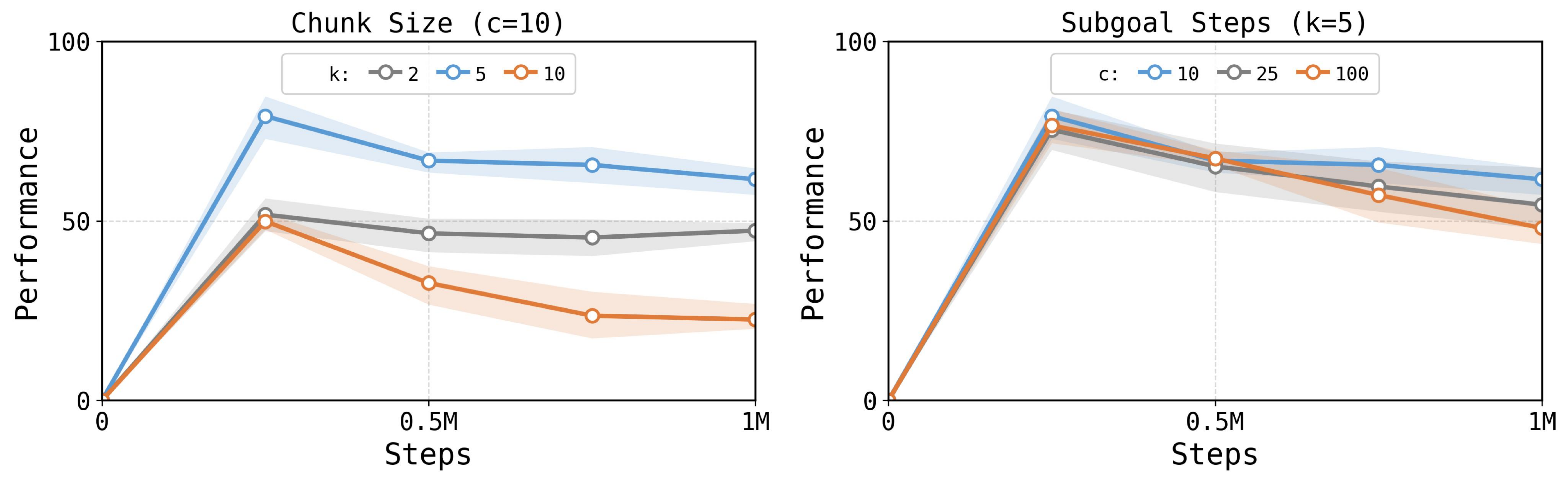}
    \caption{\textbf{Sensitivity Analysis.} Performance as a function of action chunk size $k$ (left) and high-level subgoal interval $c$ (right) evaluated on the \textit{scene-play} environment with 4 seeds.}
    \label{fig:hyperparams}
\end{figure}

\textbf{Sensitivity to Hierarchy and Chunk Size.} We analyze the sensitivity of HiQC to the action chunk size $k$ and the high-level planning horizon $c$ in Figure \ref{fig:hyperparams}. 
\begin{itemize}
    \item \textbf{Chunk Size ($k$):} Performance improves as $k$ increases from 1 to 5, consistent with the benefit of horizon compression. However, excessively large chunks ($k>10$) degrade performance, likely because open-loop execution becomes susceptible to compounding errors in the dynamics, a limitation noted in prior chunking work \citep{li2025qchunking}.
    \item \textbf{Subgoal Horizon ($c$):} We observe a ``sweet spot'' for the high-level planning horizon. If $c$ is too small, the high-level problem retains the complexity of the original task; if $c$ is too large, the low-level policy struggles to reach the distant subgoal. HiQC generally benefits from a larger $c$ than standard HIQL, as the chunked low-level policy can cover more ground per decision step.
\end{itemize}

\section{Discussion}
\label{sec:discussion}

We presented Hierarchical Implicit Q-Chunking (HiQC), a hierarchical offline goal-conditioned RL algorithm that addresses the curse of horizon \citep{park2025horizon} through a dual temporal decomposition: a high-level latent planner that decomposes long-horizon tasks into subgoals spaced by $c$ steps, and a low-level controller that executes action chunks of size $k$ to reach each subgoal. By conditioning the low-level critic on full action chunks, HiQC enables unbiased $k$-step value backups \citep{li2025qchunking}, compressing the effective horizon at both the planning and execution levels. Our theoretical analysis (Theorem~\ref{thm:hiqc}) shows that, under a bounded per-backup error model, this decomposition yields an error bound scaling of $\mathcal{O}(\sqrt{T/k})$, improving over standard two-level hierarchies \citep{park2023hiql} by a factor of $1/\sqrt{k}$ within this model. Empirically, HiQC outperforms baselines on long-horizon navigation tasks in OGBench \citep{park2025ogbench}, with its largest gains on \texttt{humanoid-giant} and \texttt{pointmaze-giant}. On high-dimensional manipulation tasks, its results are mixed: HiQC attains the best success rate on \texttt{cube-triple} but underperforms DQC on \texttt{puzzle-4x6} and \texttt{scene-play} (Supplementary Material~\ref{app:dqc_comparison}).

Several limitations remain. HiQC uses a fixed chunk size $k$ across all states, and as shown in our ablation studies (Figure~\ref{fig:hyperparams}), excessively large chunks degrade performance due to compounding open-loop errors \citep{li2025qchunking}; developing adaptive, state-conditioned chunk boundaries is an open problem shared with prior chunking methods \citep{li2025decoupled}. Furthermore, our theoretical analysis relies on the assumption of bounded per-backup errors, which may not hold strictly when combined with deep neural network function approximation; as discussed in Section~\ref{sec:theory}, we regard the analysis as intuition-building rather than an end-to-end guarantee for the trained system. On the empirical side, our results are averaged over 4 seeds per configuration (see Supplementary Material \ref{app:hyperparameters} for the underlying compute tradeoff), so results on tasks with wide confidence intervals should be interpreted with corresponding caution. Additionally, the high-level value function is unbiased only under deterministic dynamics \citep{park2023hiql}, and as noted by \citet{park2025horizon}, two-level hierarchies only mitigate rather than fundamentally solve error accumulation in TD learning. Extending HiQC to stochastic settings and deeper hierarchies remains an important direction for future work.

%% file: content/appendix.tex
\section*{Appendix}

\section{Theoretical Proofs and Derivations}
\label{app:proofs}

This appendix proves Lemma~\ref{lem:depth} and Theorem~\ref{thm:hiqc} using only unrolling and the AM--GM inequality.

\subsection{Proof of Lemma~\ref{lem:depth}}
\label{app:proof_depth}

\begin{proof}
Consider a sequence of $D$ bootstraps indexed by $j=0,1,\dots,D-1$ (these indices can represent atomic steps, chunk steps, or high-level steps depending on the method). Let $V_j$ be the exact value along this recursion and $\widehat{V}_j$ the learned value. Assume each bootstrap step satisfies
\begin{equation}
\widehat{V}_j \;=\; \text{(exact bootstrap target using $\widehat{V}_{j+1}$)} \;+\; \delta_j,
\qquad |\delta_j|\le \epsilon.
\end{equation}
Define the error $e_j \triangleq \widehat{V}_j - V_j$. Subtracting the exact recursion from the approximate recursion gives
\begin{equation}
e_j \;=\; e_{j+1} + \delta_j,
\end{equation}
where $e_D=0$ at the terminal end of the chain (no further bootstrap). Unrolling from $j=0$ yields
\begin{equation}
e_0 \;=\; \sum_{j=0}^{D-1}\delta_j
\quad\Longrightarrow\quad
|e_0| \;\le\; \sum_{j=0}^{D-1}|\delta_j|
\;\le\; D\,\epsilon.
\end{equation}
This proves \eqref{eq:depth_bound}.
\end{proof}

\subsection{Proof of Theorem~\ref{thm:hiqc}}
\label{app:proof_hiqc}

\begin{proof}
We compute bootstrap depths at each level.

\textbf{High level.} The high-level recursion advances $c$ atomic steps per bootstrap across total horizon $T$, so the bootstrap depth is $D_H=T/c$. By Lemma~\ref{lem:depth}, the high-level contribution is at most $(T/c)\epsilon_H$.

\textbf{Low level.} Within a single high-level interval of length $c$, the low-level chunk recursion advances $k$ atomic steps per bootstrap, so the bootstrap depth is $D_L=c/k$. By Lemma~\ref{lem:depth}, the low-level contribution is at most $(c/k)\epsilon_L$.

Summing gives \eqref{eq:hiqc}. To optimize over $c$, consider
\begin{equation}
f(c) \triangleq \frac{T}{c}\epsilon_H + \frac{c}{k}\epsilon_L, \qquad c>0.
\end{equation}
By AM--GM inequality, for any $c>0$,
\begin{equation}
\frac{T}{c}\epsilon_H + \frac{c}{k}\epsilon_L
\;\ge\;
2\sqrt{\left(\frac{T}{c}\epsilon_H\right)\left(\frac{c}{k}\epsilon_L\right)}
\;=\;
2\sqrt{\epsilon_H\epsilon_L}\sqrt{\frac{T}{k}}.
\end{equation}
Equality holds when the two terms are equal:
\begin{equation}
\frac{T}{c}\epsilon_H = \frac{c}{k}\epsilon_L
\quad\Longrightarrow\quad
c^\star = \sqrt{\frac{Tk\,\epsilon_H}{\epsilon_L}}.
\end{equation}
This proves \eqref{eq:c_star} and \eqref{eq:hiqc_opt}.
\end{proof}

%% file: content/supplementary.tex
\section{Implementation Details and Hyperparameters}
\label{app:hyperparameters}

We provide the full implementation details and hyperparameters for HiQC and all baselines. All algorithms were implemented in JAX. We build on top of the codebases released by \citet{park2023hiql} and \citet{park2025ogbench}.

\subsection{Compute Budget, Seeds, and Tuning}
\label{app:compute_budget}

All results in the paper are averaged over 4 random seeds. While 4 seeds is on the lower end, evaluating 9 methods across 9 environments required training over 300 models from scratch; given this computational expense, we prioritized broad environmental coverage over a larger number of seeds per task, and we note this statistical tradeoff as a limitation of our empirical study. Due to the same compute constraints, we did not tune HiQC's hyperparameters per domain; the inherited configurations are described in the following subsections.

\subsection{Shared Hyperparameters}

Table~\ref{tab:shared_params} describes the common hyperparameters shared across all methods. We follow the network architecture and optimization settings used in HIQL \cite{park2023hiql} and OGBench \cite{park2025ogbench}.

\begin{table}[ht]
    \centering
    \small
    \caption{\textbf{Shared hyperparameters across all methods.}}
    \begin{tabular}{ll}
    \toprule
    \textbf{Parameter} & \textbf{Value} \\
    \midrule
    Optimizer & Adam \\
    Learning rate & $3 \times 10^{-4}$ \\
    Batch size & $1024$ \\
    Discount factor ($\gamma$) & $0.99$ (HIQL, GCIQL, CRL, HiQC), $0.999$ (QC, DQC) \\
    Target network update rate ($\tau$) & $0.005$ \\
    Number of training steps & $10^6$ \\
    Network hidden dimensions & $(512, 512, 512)$ \\
    Nonlinearity & GELU \\
    Layer normalization & True (first two layers) \\
    \midrule
    Value goal sampling $(w_\mathrm{cur}^\mathrm{v}, w_\mathrm{traj}^\mathrm{v}, w_\mathrm{rand}^\mathrm{v})$ & $(0.2, 0.5, 0.3)$ \\
    Value geometric sampling & True \\
    Actor goal sampling $(w_\mathrm{cur}^\mathrm{p}, w_\mathrm{traj}^\mathrm{p}, w_\mathrm{rand}^\mathrm{p})$ & $(0.0, 1.0, 0.0)$ \\
    \bottomrule
    \end{tabular}
    \label{tab:shared_params}
\end{table}

For baseline methods utilizing flow matching policies (DQC, QC), the policy is parameterized by a conditional vector field $v_\omega(t, x, s, z)$ and integrated using the Euler method with 10 flow steps during inference. For HiQC and methods using Gaussian policies (GCBC, HGCBC, HGCBCAC, HIQL, GCIQL, CRL), we use a constant standard deviation parameterization for the actors, following \citet{park2023hiql}.

\subsection{HiQC Hyperparameters}

HiQC combines a high-level latent planner from HIQL \cite{park2023hiql} with a low-level chunked critic from Q-Chunking \cite{li2025qchunking}. As noted in Section~\ref{app:compute_budget}, we directly inherited the default hyperparameters from HIQL for the high-level components and from QC for the low-level chunking components, performing only limited sensitivity checks on the \texttt{scene} environments (Figure~\ref{fig:hyperparams}). The domain-specific values reported below therefore reflect these inherited configurations rather than settings tuned for HiQC. Table~\ref{tab:hiqc_shared} reports the HiQC-specific shared hyperparameters, and Table~\ref{tab:hiqc_params} reports the task-specific hyperparameters.

\begin{table}[ht]
    \centering
    \small
    \caption{\textbf{HiQC-specific hyperparameters (shared across all tasks).}}
    \begin{tabular}{ll}
    \toprule
    \textbf{Parameter} & \textbf{Value} \\
    \midrule
    Latent goal dimension ($|\mathcal{Z}|$) & $10$ \\
    High-level AWR temperature ($\alpha_H$) & $3.0$ \\
    Low-level AWR temperature ($\alpha_L$) & $3.0$ \\
    \bottomrule
    \end{tabular}
    \label{tab:hiqc_shared}
\end{table}

\begin{table}[ht]
    \centering
    \small
    \caption{\textbf{HiQC task-specific hyperparameters.} We adjust the subgoal horizon $c$, chunk size $k$, and expectile $\tau$ based on the domain. Locomotion domains use a lower expectile ($0.5$) while manipulation domains use higher expectiles ($0.9$--$0.93$).}
    \begin{tabular}{l c c c}
    \toprule
    \textbf{Domain} & \textbf{Subgoal Steps ($c$)} & \textbf{Chunk Size ($k$)} & \textbf{Expectile ($\tau$)} \\
    \midrule
    \texttt{pointmaze-*} & 25 & 5 & 0.5 \\
    \texttt{antmaze-*} & 25 & 2 & 0.5 \\
    \texttt{humanoidmaze-*} & 100 & 5 & 0.5 \\
    \texttt{cube-*} & 10 & 5 & 0.93 \\
    \texttt{scene-*} & 10 & 5 & 0.9 \\
    \texttt{puzzle-*} & 10 & 5 & 0.9 \\
    \bottomrule
    \end{tabular}
    \label{tab:hiqc_params}
\end{table}

The high-level value function is trained via Implicit V-Learning \cite{kostrikov2022offline} using $c$-step returns (Eq.~\ref{eq:vh_loss}), while the low-level critic uses an explicit Q-function with $k$-step Bellman backups (Eq.~\ref{eq:ql_loss}). We found that the explicit Q-function was important for the low-level critic to enable standard TD error minimization over action chunks. For \texttt{scene-*}, we used a reduced learning rate of $10^{-6}$, target update rate of $10^{-5}$, and low-level AWR temperature $\alpha_L = 1.0$ for training stability; these settings arose from the limited sensitivity checks noted above and are the only HiQC-specific deviations from the inherited configurations.

\subsection{Baseline Hyperparameters}

Baseline hyperparameters were selected to match the configurations reported in their respective original papers or the OGBench benchmark \cite{park2025ogbench}. For HIQL and GCIQL, we follow the hyperparameters from \citet{park2023hiql} and \citet{park2025ogbench}, respectively. For DQC and QC, we follow \citet{li2025decoupled} and \citet{li2025qchunking}. Table~\ref{tab:baseline_params} reports the task-specific hyperparameters for all baselines.

\begin{table}[ht]
    \centering
    \small
    \caption{\textbf{Task-specific hyperparameters for baselines.} For each method, we report only the parameters that deviate from the shared defaults in Table~\ref{tab:shared_params}. All other parameters use the shared defaults.}
    \setlength{\tabcolsep}{5pt}
    \scalebox{0.85}{
    \begin{tabular}{l cc cc c}
    \toprule
    & \multicolumn{2}{c}{\textbf{HIQL}} & \multicolumn{2}{c}{\textbf{GCIQL}} & \textbf{CRL} \\
    \cmidrule(lr){2-3} \cmidrule(lr){4-5} \cmidrule(lr){6-6}
    \textbf{Domain} & $c$ & $\alpha$ & $\tau$ & $\alpha$ & $\alpha$ \\
    \midrule
    \texttt{pointmaze-*} & 25 & 3.0 & 0.9 & 0.003 & 0.03 \\
    \texttt{antmaze-*} & 25 & 3.0 & 0.9 & 0.3 & 0.1 \\
    \texttt{humanoidmaze-*} & 100 & 3.0 & 0.9 & 0.1 & 0.1 \\
    \bottomrule
    \end{tabular}}
    \label{tab:baseline_params}
\end{table}

\begin{table}[ht]
    \centering
    \small
    \caption{\textbf{Task-specific hyperparameters for action chunking baselines.} DQC and QC use $\gamma = 0.999$ and flow matching policies. HGCBC and HGCBCAC use Gaussian policies.}
    \setlength{\tabcolsep}{3pt}
    \scalebox{0.82}{
    \begin{tabular}{l cccc cc cc}
    \toprule
    & \multicolumn{4}{c}{\textbf{DQC}} & \multicolumn{2}{c}{\textbf{QC}} & \multicolumn{2}{c}{\textbf{HGCBCAC}} \\
    \cmidrule(lr){2-5} \cmidrule(lr){6-7} \cmidrule(lr){8-9}
    \textbf{Domain} & $h$ & $h_a$ & $\kappa_b$ & $\kappa_d$ & $h{=}h_a$ & $\kappa_b$ & $c$ & $k$ \\
    \midrule
    \texttt{pointmaze-*} & 25 & 5 & 0.9 & 0.5 & 5 & 0.9 & 25 & 5 \\
    \texttt{antmaze-*} & 25 & 5 & 0.9 & 0.5 & 5 & 0.9 & 25 & 2 \\
    \texttt{humanoidmaze-*} & 25 & 1 & 0.5 & 0.8 & 5 & 0.5 & 100 & 5 \\
    \texttt{cube-triple} & 25 & 5 & 0.93 & 0.8 & 5 & 0.93 & 10 & 5 \\
    \texttt{scene-*} & 25 & 5 & 0.9 & 0.5 & 5 & 0.9 & 10 & 5 \\
    \texttt{puzzle-4x6} & 25 & 5 & 0.7 & 0.5 & 5 & 0.7 & 10 & 5 \\
    \bottomrule
    \end{tabular}}
    \label{tab:chunking_baseline_params}
\end{table}

\textbf{HGCBC / HGCBCAC.} These hierarchical behavioral cloning baselines use the same high-level architecture as HiQC (latent subgoal prediction with AWR) but lack the value-based planning component. HGCBC predicts single actions at the low level, while HGCBCAC extends HGCBC to predict action chunks. Both use Gaussian policy parameterizations with the same subgoal horizons as HiQC (Table~\ref{tab:hiqc_params}). For HGCBCAC, the chunk size is $k=2$ for \texttt{antmaze-*} and $k=5$ for all other domains.

\textbf{GCBC.} Goal-Conditioned Behavioral Cloning uses a Gaussian policy with the same network architecture as the other methods. No value function is trained.

\section{Comparison with DQC on Manipulation Tasks}
\label{app:dqc_comparison}

In Table~\ref{tab:main_results}, Decoupled Q-Chunking (DQC) \citep{li2025decoupled} outperforms HiQC on \texttt{puzzle-4x6} ($16\%$ vs.\ $1\%$) and \texttt{scene-play} ($76\%$ vs.\ $65\%$), while the two are comparable on \texttt{cube-triple} ($8\%$ vs.\ $9\%$). Here we analyze which structural differences between the two methods are consistent with the gap.

\textbf{Structural differences.} The configurations in Tables~\ref{tab:hiqc_params} and \ref{tab:chunking_baseline_params} isolate two differences on these tasks:
\begin{itemize} \setlength\itemsep{0em}
    \item \textbf{Critic horizon.} DQC's critic performs unbiased backups over $h=25$ atomic steps while its policy executes $h_a=5$ actions. HiQC's low-level critic backs up only over its executed chunk ($k=5$) and covers longer ranges through the hierarchy ($c=10$). DQC thus propagates value over a longer span per backup than either level of HiQC's decomposition.
    \item \textbf{Latent subgoal bottleneck.} HiQC routes long-range information through a $10$-dimensional latent subgoal $\phi(s)$ (Table~\ref{tab:hiqc_shared}), and its low-level policy targets the latent subgoal rather than the task goal. DQC conditions directly on the task goal.
\end{itemize}
Both methods execute $5$-action chunks open-loop on these tasks, so executed chunk length alone is unlikely to explain the gap.

\textbf{Hypothesized mechanisms.} We hypothesize that DQC's advantage stems from (i) its longer decoupled critic horizon and (ii) the absence of a latent bottleneck, which may under-specify the discrete object configurations (e.g., button states in \texttt{puzzle}) that define progress; if latent subgoals fail to encode object state, the low-level policy receives an uninformative target regardless of execution quality. HiQC also required reduced learning rates for stable training on \texttt{scene-*} (Section~\ref{app:hyperparameters}), indicating the chunked low-level critic is harder to optimize in these domains. These explanations remain hypotheses; attributing the gap definitively would require a controlled decomposition (e.g., varying DQC's critic horizon or replacing HiQC's latent subgoals with raw goal states), which we leave to future work.